\documentclass{article}

\usepackage[preprint,nonatbib]{neurips_2019}

\usepackage[utf8]{inputenc}
\usepackage[T1]{fontenc}
\usepackage{enumitem}
\usepackage{url}
\usepackage{amsfonts}
\usepackage{nicefrac}
\usepackage{microtype}
\usepackage{graphicx}
\usepackage{subfigure}
\usepackage{booktabs} 
\usepackage{hyperref}
\usepackage{array}
\usepackage{amsthm}
\usepackage{amssymb}
\usepackage{algorithm}
\usepackage{algorithmic}
\usepackage{lineno}
\usepackage{amsmath}
\usepackage{colortbl}
\usepackage{multirow}
\usepackage{setspace}
\usepackage[backend=biber,style=numeric,citestyle=numeric,sorting=none]{biblatex}
\usepackage{pifont}
\newcommand{\cmark}{\ding{51}}
\newcommand{\xmark}{\ding{55}}
\newtheorem{prop}{Proposition}
\graphicspath{{images/}}

\title{Exploiting Channel Similarity for Accelerating Deep Convolutional Neural Networks}

\author{Yunxiang Zhang$^{1}$\thanks{Equal contribution.}, \, \textbf{Chenglong Zhao}$^{1}$\textsuperscript{$\ast$}, \textbf{Bingbing Ni}$^{1}$, \textbf{Jian Zhang}$^{1}$, \textbf{Haoran Deng}$^{1}$ \\ $^{1}$Shanghai Jiao Tong University, Shanghai, China \\
yunxiang.zhang.polytechnique@gmail.com \\
cl-zhao@sjtu.edu.cn}

\begin{document}

\maketitle

%%%%%%%%%%%%%%%%%%%%%%%%%%%%%%%%%%%%%%%%%%%%%%%%%%%%%%%%

\begin{abstract}
To address the limitations of existing magnitude-based pruning algorithms in cases where model weights or activations are of large and similar magnitude, we propose a novel perspective to discover parameter redundancy among channels and accelerate deep CNNs via channel pruning. Precisely, we argue that channels revealing similar feature information have functional overlap and that most channels within each such similarity group can be removed without compromising model's representational power. After deriving an effective metric for evaluating channel similarity through probabilistic modeling, we introduce a pruning algorithm via hierarchical clustering of channels. In particular, the proposed algorithm does not rely on sparsity training techniques or complex data-driven optimization and can be directly applied to pre-trained models. Extensive experiments on benchmark datasets strongly demonstrate the superior acceleration performance of our approach over prior arts. On ImageNet, our pruned ResNet-50 with 30\% FLOPs reduced outperforms the baseline model.
\end{abstract}

%%%%%%%%%%%%%%%%%%%%%%%%%%%%%%%%%%%%%%%%%%%%%%%%%%%%%%%%

\section{Introduction}

Ever since the renaissance of convolutional neural networks (CNNs) a decade ago, the general trend has been pushing them deeper and deeper to achieve better performance~\cite{alexnet,vgg,resnet}, along with a drastic increase in computation and storage requirements for them. However, it is no easy thing to effectively deploy these unprecedentedly large models in practical applications. Additionally, recent works~\cite{denil2013predicting,ba2014do} have been making further efforts to confirm the over-parameterization problem lying in deep CNNs, revealing that these powerful models do not necessarily have to be so cumbersome.

To bridge the gap between limited computational resources and superior performance of deep CNNs, various model acceleration algorithms have been proposed, including network pruning~\cite{liu2017learning,he2017channel}, low-rank factorization~\cite{jaderberg2014speeding,denton2014exploiting}, network quantization~\cite{courbariaux2016binarized,rastegari2016xnor}, and knowledge distillation~\cite{hinton2015distilling,romero2014fitnets}, among which network pruning is of active research interests.

Despite their favorable performance in certain cases, current magnitude-based approaches to network pruning have some inherent limitations and cannot guarantee stable behaviors in general situations. As pointed out in~\cite{he2019pruning}, magnitude-based methods rely on two preconditions to achieve satisfactory performance, i.e. large magnitude deviation and small minimum magnitude. Without the guarantee of the two prerequisites, these methods are prone to mistakenly removing some weights that are crucial to network's performance. Unfortunately, these requirements are not always met, and we empirically observe that the performance of magnitude-based pruning algorithms deteriorates rapidly under those adverse circumstances.    

In order to address this important issue and derive a more general approach to accelerating deep CNNs, we propose to discover parameter redundancy among feature channels from a novel perspective. To be precise, we argue that channels revealing similar feature representations have functional overlap and that most channels within each such similarity group can be discarded without compromising network's representational power. The proposed similarity-based approach is more general in that it remains applicable beyond the restricted scenarios of its magnitude-based counterpart, as will be demonstrated through our experiments. Figure~\ref{fig:motivation} shows a graphical illustration of our motivation.

Our contributions can be summarized as follows:
\vspace{-2mm}
\begin{itemize}[leftmargin=7mm]
\item Propose a novel perspective, i.e. similarity-based approach, for channel-level pruning. 
\item Introduce an effective metric for evaluating channel similarity via probabilistic modelling.
\item Develop an efficient channel pruning algorithm via hierarchical clustering based on this metric.
\end{itemize}

%%%%%%%%%%%%%%%%%%%%%%%%%%%%%%%%%%%%%%%%%%%%%%%%%%%%%%%%%%%%%%%%%%%%

\section{Related work}

\textbf{Non-structured pruning.}
Earlier works on network pruning~\cite{lecun1990optimal,hassibi1993second} remove redundant parameters by analyzing the Hessian matrix of loss function. Several magnitude-based methods~\cite{han2015learning,han2015deep,guo2016dynamic} drop network weights with insignificant values to perform model compression.~\cite{dai2018compressing,molchanov2017} leverage Bayesian theory to achieve more interpretable pruning. Similar to our approach,~\cite{srinivas2015data,mariet2015diversity} also exploit the notion of similarity for reducing parameter redundancy. However, they only consider fully-connected layers and obtain limited acceleration.~\cite{son2018clustering} introduces k-means clustering to extract the centroids of kernels and achieves computation reduction via kernel sharing.

\textbf{Structured Pruning.} 
To avoid the requirement of specific hardwares and libraries for sparse matrix operations, various channel-level pruning algorithms were proposed.~\cite{wen2016learning,lebedev2016fast,huang2018data,zhou2016less} impose carefully designed sparsity constraints on network weights during training to facilitate the removal of redundant channels. These algorithms rely on training networks from scratch with sparsity and cannot be directly applied to accelerating pre-trained ones.~\cite{he2017channel,luo2017thinet,zhuang2018discrimination} transform network pruning into an optimization problem. Analogous to our work,~\cite{roychowdhury2017reducing} investigates the presence of duplicate filters in neural networks and finds that some of them can be reduced without impairing network's performance.

\textbf{Neural Architecture Search.} 
While compact CNN models~\cite{zhang2018shufflenet, howard2017mobilenets} were mostly designed in hand-crafted manner, neural architecture search has also shown its potential in discovering efficient models. Some methods leverage reinforcement learning~\cite{baker2017designing,barret2017neural} or evolutionary algorithms~\cite{real2017large,liu2018hierarchical} to conduct architecture search in discrete spaces, the others~\cite{luo2018neural,liu2018darts} perform optimization in continuous ones.

%%%%%%%%%%%%%%%%%%%%%%%%%%%%%%%%%%%%%%%%%%%%%%%%%%%%%%%%%%%%%%%%%%%%

\begin{figure}[t]
\centering
\includegraphics[width=0.9\linewidth]{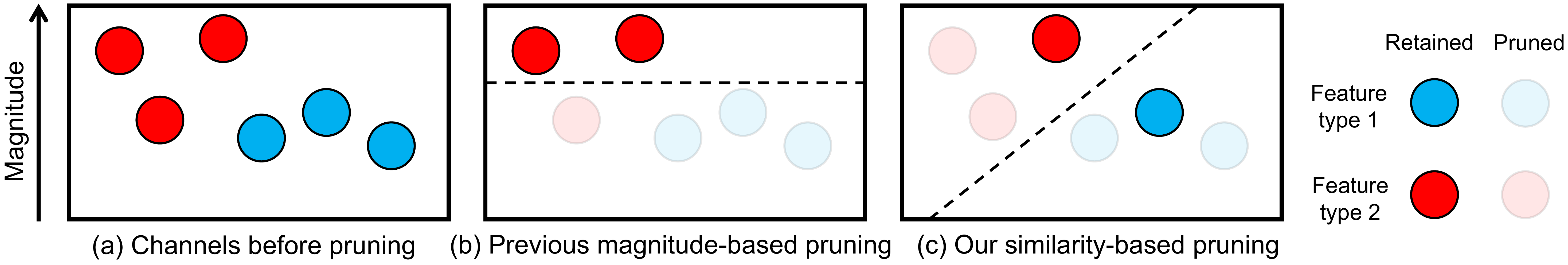}
\caption{Magnitude-based approach vs similarity-based approach to channel pruning.}
\label{fig:motivation}
\end{figure}

%%%%%%%%%%%%%%%%%%%%%%%%%%%%%%%%%%%%%%%%%%%%%%%%%%%%%%%%%%%%%%%%%%%%

\section{Network pruning via channel similarity}
\label{sec:method}

% Unlike most previous works on channel pruning, which focus on designing appropriate importance criteria and determine redundant channels based on their relative importance, we propose to discover parameter redundancy among feature channels from a more general perspective. Precisely, we show that channels revealing similar feature information have functional overlap and that most channels within each such similarity group can be removed without compromising the representational power of the network. Note that our interpretation of redundancy is more general in that it covers the case where unimportant channels are considered superfluous, as these channels are also very similar.

% In this section, we first introduce a \textit{channel similarity} metric to effectively measure the discrepancy among feature channels, then leverage the statistical information provided by batch normalization layers to efficiently compute this quantity. After that, a theoretical upper bound is derived to support our similarity-based pruning approach. Finally, this section concludes with the detailed workflow of the proposed one-shot channel pruning algorithm.

%%%%%%%%%%%%%%%%%%%%%%%%%%%%%%%%%%%%%%%%%%%%%%%%%%%%%%%%

\subsection{Channel similarity}
\label{sec:metric}

Let us denote the activations of an arbitrary convolutional layer by $\mathcal{A} = \{\mathcal{A}^{(c)} \, | \, c \in \{1, ..., C\}\} \in \mathbb{R}^{C \times H \times W \times B}$, where $H, W, C, B$ represent the height \& width of feature maps, the number of feature channels, and the batch size, respectively. Each channel can thus be represented by a 3-dimensional tensor $\mathcal{A}^{(c)} \in \mathbb{R}^{H \times W \times B}$, with $c$ its corresponding channel index.

Intuitively, if we want to quantify the similarity between two feature channels, \textit{mean squared error} (MSE) provides a simple yet effective metric, which is commonly adopted to evaluate to what extent a tensor deviates from another one. Therefore, using previously introduced notations, the distance, or dissimilarity, between two feature channels $\mathcal{A}^{(i)}$ and $\mathcal{A}^{(j)}$ can be defined as follows:
\begin{equation}
\label{eq:distance}
Dist( \, \mathcal{A}^{(i)}, \mathcal{A}^{(j)}) = (H \times W \times B)^{-1} \times \lVert \mathcal{A}^{(i)} - \mathcal{A}^{(j)}\rVert_{2}^{2}
\end{equation}
However, naively computing channel distance in this manner has two obvious limitations. First, this definition is not consistent across different data batches, as the activation values depend on the input samples fed to the network. Second, computing the distance between all possible pairs of feature channels via this formulation is computationally inefficient in view of its $\mathcal{O}(H \times W \times B \times C^{2})$ time complexity, not to speak of the fact that modern CNNs generally contain millions of activations.

To address these issues, we propose a probabilistic approach to estimating channel distance. Precisely, each activation is represented by a random variable $\mathcal{A}_{h, w, b}^{(c)} \sim \mathcal{D}_{h, w, b}^{(c)}( \, \mu_{h, w, b}^{(c)}, (\sigma_{h, w, b}^{(c)})^{2})$ with mean $\mu_{h, w, b}^{(c)}$ and variance $(\sigma_{h, w, b}^{(c)})^{2}$, where $h, w, b$ index the height \& width of feature maps and the position in current batch, respectively.

\begin{prop}
\label{prop1}
Assume that the activations belonging to the same channel are i.i.d. and that any two activations from two different channels are mutually independent. For any two channels in the same convolutional layer, the distance between them converges in probability to the following value: 
\begin{equation}
\frac{1}{n} \lVert \mathcal{A}^{(i)} - \mathcal{A}^{(j)} \rVert_{2}^{2} \xrightarrow[n \to \infty]{prob} ( \, \mu^{(i)} - \mu^{(j)})^{2} + (\sigma^{(i)})^{2} + (\sigma^{(j)})^{2} \quad \text{where} \, n = H \times W \times B
\end{equation}
\end{prop}
\begin{proof}
Detailed derivation can be found in Appendix B of the supplementary material.
% The proof of this proposition is based on the weak law of large numbers. Detailed derivation can be found in the accompanying supplementary material.
\end{proof}

The proposition above provides a nice approximation to the distance between two feature channels under our probabilistic settings, since $n$ generally reaches up to tens of thousands in modern CNNs. While the assumption of mutual independence may appear rather strong at first glance, we empirically validate its reasonability in Section~\ref{sec:estimation} by showing that the proposed probabilistic approach induces no bias to channel distance estimation compared to its naive activation-value-based counterpart.

Now if we reexamine the time complexity of this novel formulation for computing channel distance, it drops down to $\mathcal{O}(C^{2})$. In real-life applications, for a batch of 64 normal images of size $2000 \times 1500$ (e.g. smartphone), the saving in computing time can be considerable (up to $10^{8}$ for lower layers). In short, this probabilistic approach is scalable to all kinds of application scenarios given the fact that the number of channels $C$ rarely exceeds 1000 in existing CNN architectures. Now that we have instantiated the notion of \textit{channel similarity} via the proposed \textit{channel distance},\footnote{If the distance between two channels $\mathcal{A}^{(i)}$ and $\mathcal{A}^{(j)}$ is smaller than that between $\mathcal{A}^{(i)}$ and another channel $\mathcal{A}^{(k)}$, then $\mathcal{A}^{(i)}$ is considered more similar to $\mathcal{A}^{(j)}$ than to $\mathcal{A}^{(k)}$.} the next step is to find an efficient way to derive the statistical information of the activations in each feature channel.

%%%%%%%%%%%%%%%%%%%%%%%%%%%%%%%%%%%%%%%%%%%%%%%%%%%%%%%%

\subsection{Channel similarity via batch normalization}
\label{sec:bn}

\textit{Batch normalization} (BN)~\cite{batchnorm} has been introduced to enable faster and more stable training of deep CNNs and is now becoming an indispensable component in deep learning. The composite of three consecutive operations: linear convolution, batch normalization, and \textit{rectified linear unit} (ReLU)~\cite{relu}, is widely adopted as building blocks in state-of-the-art deep CNN architectures~\cite{resnet, densenet}.

The way BN normalizes the activations within a feature channel motivates us to base our probabilistic settings on the statistical information in the outputs of BN layers. In particular, BN normalizes the activations using mini-batch statistics, which perfectly matches our definition of channel distance.

Let us denote the $c$-th input and output feature maps of a BN layer by $x$ and $y$, respectively. For an arbitrary mini-batch $\mathcal{B}$, this BN layer performs the following transformation:
\begin{equation}
y_{h, w, b} = {\rm BN}_{\gamma, \beta}(x_{h, w, b}) = \frac{\gamma}{\sqrt{\sigma_{\mathcal{B}}^{2} + \epsilon}} \, (x_{h, w, b} - \mu_{\mathcal{B}}) + \beta
\end{equation}
where $\mu_{\mathcal{B}}$ and $\sigma_{\mathcal{B}}^{2}$ denote the mini-batch mean and variance, $\gamma$ and $\beta$ are two trainable parameters of an affine transformation that helps to restore the representational power of the network. $\epsilon$ is a small positive number added to the mini-batch variance for numerical stability.  

Given the fact that the overwhelming majority of modern CNN architectures adopt the convention of inserting a BN layer after each convolutional layer, it would be very convenient for us to directly leverage the statistical information provided by BN layers for channel distance estimation. Under the probabilistic settings established in Section~\ref{sec:metric}, the batch-normalized activations of the $c$-th feature channel are i.i.d. random variables with mean $\beta^{(c)}$ and variance $(\gamma^{(c)})^{2}$, i.e. ${\rm BN}_{\gamma^{(c)}, \beta^{(c)}}( \, \mathcal{A}_{h, w, b}^{(c)}) \sim \mathcal{D}^{(c)}( \, \beta^{(c)}, (\gamma^{(c)})^{2}), \hspace{1mm} \forall h, w, b$. Consequently, we can straightforwardly compute the distance between two batch-normalized feature channels in the same convolutional layer as follows:
\begin{equation}
\label{eq:distance-bn}
\begin{array}{c}
Dist( \, \mathcal{N}^{(i)}, \mathcal{N}^{(j)}) \backsimeq ( \, \beta^{(i)} - \beta^{(j)})^{2} + (\gamma^{(i)})^{2} + (\gamma^{(j)})^{2} \\ [1mm]
\text{where} \hspace{1mm} \mathcal{N}^{(i)} = {\rm BN}_{\gamma^{(i)}, \beta^{(i)}}( \, \mathcal{A}^{(i)}) \hspace{1mm} \text{and} \hspace{1mm} \mathcal{N}^{(j)} = {\rm BN}_{\gamma^{(j)}, \beta^{(j)}}( \, \mathcal{A}^{(j)})
\end{array}
\end{equation}

%%%%%%%%%%%%%%%%%%%%%%%%%%%%%%%%%%%%%%%%%%%%%%%%%%%%%%%%

\subsection{From similarity to redundancy}
\label{sec:theory}

Existing magnitude-based pruning algorithms rely on the argument that removing those parameters with relatively insignificant values will incur little impact on the network's performance. However, as pointed out in~\cite{lecun1990optimal, hassibi1993second}, this intuitive idea does not seem to be theoretically well-founded. In contrast, we derive a theoretical support to justify the reasonability of our similarity-based pruning approach. In particular, we show that the removal of an arbitrary feature channel will not impair the network's representational power in a dramatic way, as long as there exists another channel that is sufficiently similar to the removed one and can be exploited as a substitution.

Let us consider two consecutive convolutional layers $\mathcal{A}^{(l)}$ and $\mathcal{A}^{(l+1)}$, we have $\mathcal{A}^{(l)} = \{\mathcal{A}^{(l, c_{l})} \, | \, c_{l} \in \{1, ..., C_{l}\}\} \in \mathbb{R}^{C_{l} \times H_{l} \times W_{l} \times B}$ using similar notations as before. Suppose that batch normalization ${\rm BN}_{\gamma, \beta}$ and non-linear activation $h$ are applied after each linear convolution, then we have:
\begin{equation}
\begin{array}{c}
\mathcal{A}^{(l+1, c_{l+1})} = \sum_{c_{l} \in \{1, ..., C_{l}\}} h(\mathcal{N}^{(l, c_{l})}) \ast \mathcal{W}^{(c_{l}, c_{l+1})} \\ [2mm]
\mathcal{N}^{(l, c_{l})} = {\rm BN}_{\gamma^{(c_{l})}, \beta^{(c_{l})}}(\mathcal{A}^{(l, c_{l})})
\end{array}
\end{equation}
where $\mathcal{W}^{(c_{l}, c_{l+1})} \in \mathbb{R}^{K \times K}$ represents the $c_{l}$-th kernel matrix in the $c_{l+1}$-th convolutional filter and $\ast$ denotes the convolution operation. Note that BN nullifies the effect of bias vectors in convolutional layers, hence they are deprecated in the formulation above.

Inspired by~\cite{srinivas2015data}, we explore and analyze to what extent the activations of $\mathcal{A}^{(l+1)}$ will be shifted if we remove a feature channel from $\mathcal{N}^{(l)}$ and compensate the consequent loss of representational power by exploiting the channels in $\mathcal{N}^{(l)}$ that are similar to the removed one. Suppose that $\mathcal{N}^{(l, i)}$ and $\mathcal{N}^{(l, j)}$ are two similar feature channels in $\mathcal{N}^{(l)}$, now if we remove the former and properly update the kernel matrix corresponding to the latter in an attempt to minimize the resulting performance decay, then for each feature channel $\mathcal{A}_{p}^{(l+1, c_{l+1})}$ after pruning, we have:
\begin{equation}
\mathcal{A}_{p}^{(l+1, c_{l+1})} = h(\mathcal{N}^{(l, j)}) \ast (\mathcal{W}^{(i, c_{l+1})} + \mathcal{W}^{(j, c_{l+1})}) + \sum_{c_{l} \neq i, j} h(\mathcal{N}^{(l, c_{l})}) \ast \mathcal{W}^{(c_{l}, c_{l+1})}
\end{equation}
Note that we replace the kernel matrix $\mathcal{W}^{(j, c_{l+1})}$ by $\mathcal{W}^{(i, c_{l+1})} + \mathcal{W}^{(j, c_{l+1})}$, which is a simple and intuitive way to compensate the loss of representational power resulted by pruning $\mathcal{N}^{(l, i)}$. Computing the distance between $\mathcal{A}^{(l+1, c_{l+1})}$ and $\mathcal{A}_{p}^{(l+1, c_{l+1})}$ using Equation~\ref{eq:distance} gives:
\begin{equation}
\label{eq:shift}
Dist( \, \mathcal{A}^{(l+1, c_{l+1})}, \mathcal{A}_{p}^{(l+1, c_{l+1})}) = \frac{1}{n_{l+1}} \lVert (h(\mathcal{N}^{(l, i)}) - h(\mathcal{N}^{(l, j)})) \ast \mathcal{W}^{(i, c_{l+1})} \rVert_{2}^{2}
\end{equation}
where $n_{l+1} = H_{l+1} \times W_{l+1} \times B$.

\begin{prop}
\label{prop2}
For each feature channel $\mathcal{A}^{(l+1, c_{l+1})}$ in the $(l+1)$-th convolutional layer, the distance shift caused by removing the feature channel $\mathcal{N}^{(l, i)}$ from the $l$-th convolutional layer, as defined in Equation~\ref{eq:shift}, admits the following upper bound:
\begin{equation}
Dist( \, \mathcal{A}^{(l+1, c_{l+1})}, \mathcal{A}_{p}^{(l+1, c_{l+1})}) \le \lambda \times \min_{j \in \{1, ..., C_{l}\}} Dist( \, \mathcal{N}^{(l, i)}, \mathcal{N}^{(l, j)})
\end{equation}
where $\lambda = \frac{n_{l}}{n_{l+1}} K^{2} \lVert \mathcal{W}^{(i, c_{l+1})} \rVert_{2}^{2}$ and $K^{2}$ corresponds to the size of each kernel matrix $\mathcal{W}^{(c_{l}, c_{l+1})}$.
\end{prop}
\begin{proof}
Detailed derivation can be found in Appendix B of the supplementary material.
% The proof of this proposition is based on Cauchy-Schwarz inequality. Detailed derivation can be found in the accompanying supplementary material.
\end{proof}

Through the conclusion of Proposition~\ref{prop2}, we can notice that $\lambda$ depends on the size of feature channels, the size and $L_{2}$ norm of the kernel matrix. In practice, the coefficient $\lambda$ is typically a small value of magnitude $\sim 10^{-2}$, which means that the removal of a feature channel results in rather limited shift on the next layer's activations and hence has little impact on the network's representational power, as long as there exists adequately similar channels to replace its function. Note that the above-mentioned one-to-one substitution strategy is mainly introduced for theoretical concern. In practical applications, it is clearly a sub-optimal solution from an overall perspective. For instance, we could have exploited more channels to substitute for the pruned one, i.e. a linear combination. In our implementation, the retained kernel matrices after pruning are adaptively updated by gradient-descent-based optimization algorithms rather than manually computed in a fixed way, as this automatic approach normally returns a superior updating strategy. Experimental results well reflect the effectiveness of this choice.

%%%%%%%%%%%%%%%%%%%%%%%%%%%%%%%%%%%%%%%%%%%%%%%%%%%%%%%%

\subsection{Similarity-based pruning via hierarchical clustering}
\label{sec:algorithm}

Now that we have properly defined our metric of channel similarity and have theoretically demonstrated the feasibility of our similarity-based pruning approach, it is very natural to resort to clustering algorithms for the subsequent pruning process. Precisely, we want to group the channels within each convolutional layer into similarity clusters based on its channel distance matrix.\footnote{A square matrix that groups the distance between all possible pairs of channels within a convolutional layer.} Given the results of clustering, we only need to retain one representative channel for each cluster, as the representation information provided by the others is highly similar and hence redundant. Here we select the channel with the largest $|\gamma|$ to enable faster and easier fine-tuning process after the pruning operation.  

To this end, \textit{hierarchical clustering} (HC) is introduced here as our agglomerative method. Compared with other popular clustering algorithms, HC is more adaptive to our demand as it requires only one hyper-parameter, i.e. the threshold distance. In particular, this threshold allows us to simultaneously control the clustering result and hence the pruning ratio of all layers with a single global parameter. This property is crucial to discovering efficient and compact CNN models, as the target architecture is automatically determined by the pruning algorithm~\cite{liu2018rethinking}. In order to render the distance values comparable across all layers and make our one-shot pruning more stable, we further normalize the values of each distance matrix to $[0, 1]$ before feeding them to the clustering algorithm. Empirical results in Section~\ref{sec:normalization} showcase the importance of this step. 

The overall workflow of our similarity-based pruning algorithm can be summarized as follows:
\vspace{-2mm}
\begin{itemize}
    \item Construct a channel distance matrix $\mathcal{D}^{(l)}$ for each convolutional layer using Equation~\ref{eq:distance-bn}. \item Normalize each channel distance matrix: $\mathcal{D}^{(l)} = (\mathcal{D}^{(l)} - \mathcal{D}_{min}^{(l)}) \times (\mathcal{D}_{max}^{(l)} - \mathcal{D}_{min}^{(l)})^{-1}$.
    \item Perform hierarchical clustering of channels with global threshold $t$ on each convolutional layer using its normalized channel distance matrix obtained from step 2. 
    \item For each cluster, retain the channel with the largest $|\gamma|$ and remove the others.
    \item Fine-tune the resulting pruned model for a few epochs to restore performance.
\end{itemize}
\vspace{-2mm}
Corresponding pseudo code can be found in Appendix C of the supplementary material.

%%%%%%%%%%%%%%%%%%%%%%%%%%%%%%%%%%%%%%%%%%%%%%%%%%%%%%%%%%%%%%%%%%%%

\begin{table*}[t!]
\caption{Acceleration results on CIFAR and ImageNet datasets. ``Base'' represents the uncompressed baseline model obtained via normal training, ``$t$'' denotes the threshold used for controlling pruning ratio, and ``--'' means that the result is meaningless or not available. In ``Automatic'' column, ``\cmark'' and ``\xmark'' indicate whether the target compact architecture is automatically discovered by the pruning algorithm or manually pre-defined. In ``Pre-train'' column, ``\cmark'' and ``\xmark'' indicate whether the algorithm can be directly applied to pruning pre-trained models or not. On CIFAR dataset, we report the accuracy on test set. On ImageNet dataset, single view evaluation of Top-1 and Top-5 \textit{accuracy drop} is reported to accommodate the discrepancy among different deep learning frameworks. The overall FLOPs and wall-clock time in (c) correspond to a data batch of size 256.}
    \vskip 1mm
	\centering 
	(a) Comparison of acceleration performance on CIFAR.
	\vskip 1mm
	\begin{tabular}{p{1.6cm}<{\centering} p{2.5cm}<{\centering} p{1.3cm}<{\centering} p{1.3cm}<{\centering} p{1.3cm}<{\centering} p{1.3cm}<{\centering} p{1.3cm}<{\centering}}
		\toprule 
	    Dataset	& Model & Automatic & Pre-train & Accuracy & FLOPs & Pruned \\
		\midrule 
		\multirow{10}{*}{CIFAR-10}
		& VGG-16 Base & -- & -- & 93.39\% & 627.36M & -- \\
        \specialrule{0em}{1pt}{1pt}
		\cline{2-7}
	    \specialrule{0em}{1pt}{1pt}
	    & Ours ($t = 0.25$) & \cmark & \cmark & \textbf{92.71\%} & \textbf{182.31M} & \textbf{70.94\%} \\
	    & NS~\cite{liu2017learning} & \cmark & \xmark & 92.33\% & 182.47M & 70.92\% \\
        & SSL~\cite{wen2016learning} & \cmark & \xmark & 91.65\% & 248.94M & 60.32\% \\
        & RDF~\cite{roychowdhury2017reducing} & \cmark & \cmark & 91.84\% & 314.33M & 49.90\% \\
        \specialrule{0em}{1pt}{1pt}
		\cline{2-7}
	    \specialrule{0em}{1pt}{1pt}
	    & ResNet-110 Base & -- & -- & 94.65\% & 341.23M & -- \\
        \specialrule{0em}{1pt}{1pt}
		\cline{2-7}
	    \specialrule{0em}{1pt}{1pt}
	    & Ours ($t = 0.225$) & \cmark & \cmark & \textbf{94.25\%} & \textbf{141.27M} & \textbf{58.60\%} \\
	    & NS~\cite{liu2017learning} & \cmark & \xmark & 93.57\% & 141.77M & 58.45\% \\
        % & SSL~\cite{wen2016learning} & \cmark & \xmark & 93.65\% & 141.83M & 58.44\% \\
        & PFGM~\cite{he2019pruning} & \xmark & \cmark & 93.73\% & 177.38M & 48.02\% \\
        & SFP~\cite{he2018soft} & \xmark & \cmark & 93.38\% & 217.38M & 36.30\% \\
        \midrule
		\multirow{10}{*}{CIFAR-100}
		& VGG-16 Base & -- & -- & 72.10\% & 627.45M & -- \\
        \specialrule{0em}{1pt}{1pt}
		\cline{2-7}
	    \specialrule{0em}{1pt}{1pt}
	    & Ours ($t = 0.20$) & \cmark & \cmark & \textbf{71.22\%} & \textbf{334.16M} & \textbf{46.74\%} \\
	    & NS~\cite{liu2017learning} & \cmark & \xmark & 68.82\% & 441.89M & 29.57\% \\
        & SSL~\cite{wen2016learning} & \cmark & \xmark & 68.80\% & 374.10M & 40.38\% \\
        & RDF~\cite{roychowdhury2017reducing} & \cmark & \cmark & 69.63\% & 388.53M & 38.08\% \\
        \specialrule{0em}{1pt}{1pt}
		\cline{2-7}
	    \specialrule{0em}{1pt}{1pt}
	    & ResNet-110 Base & -- & -- & 75.62\% & 341.28M & -- \\
        \specialrule{0em}{1pt}{1pt}
		\cline{2-7}
	    \specialrule{0em}{1pt}{1pt}
	    & Ours ($t = 0.19$) & \cmark & \cmark & \textbf{74.56\%} & \textbf{153.73M} & \textbf{54.96\%} \\
	    & NS~\cite{liu2017learning} & \cmark & \xmark & 73.67\% & 153.96M & 54.89\% \\
        & SSL~\cite{wen2016learning} & \cmark & \xmark & 71.28\% & 161.95M & 52.55\% \\
        & RDF~\cite{roychowdhury2017reducing} & \cmark & \cmark & 71.52\% & 190.03M & 44.32\% \\
		\bottomrule \\
	\end{tabular}
    
    (b) Comparison of acceleration performance on ImageNet for ResNet-50.
	\vskip 1mm
	\begin{tabular}{p{1.3cm}<{\centering} p{2.2cm}<{\centering} p{1.3cm}<{\centering} p{1.3cm}<{\centering} p{1.1cm}<{\centering} p{1.1cm}<{\centering} p{1.1cm}<{\centering} p{1.1cm}<{\centering}}        
	    \toprule
	    Dataset	& Model & Automatic & Pre-train & Top-1 & Top-5 & FLOPs & Pruned \\ 
		\midrule 
		\multirow{6}{*}{ImageNet}
	    & Ours ($t = 0.10$) & \cmark & \cmark & \textbf{-0.80\%} & \textbf{-0.34\%} & \textbf{4.44B} & \textbf{45.90\%} \\
        & TN~\cite{luo2017thinet} & \xmark & \cmark & -3.26\% & -1.53\% & 5.23B & 36.27\% \\
        & NS~\cite{liu2017learning} & \cmark & \xmark & -2.16\% & -1.17\% & 4.49B & 45.37\% \\
        & CP~\cite{he2017channel} & \xmark & \cmark & -3.25\% & -- & 5.49B & 33.13\% \\
        & PFGM~\cite{he2019pruning} & \xmark & \cmark & -1.32\% & -0.55\% & 4.46B & 45.74\% \\
        & SFP~\cite{he2018soft} & \xmark & \cmark & -1.54\% & -0.81\% & 5.23B & 36.27\% \\
		\bottomrule \\
	\end{tabular}
	
	(c) Wall-clock time saving of pruned ResNet-50 on ImageNet.
	\vskip 1mm
	\begin{tabular}{p{1.3cm}<{\centering} p{2.2cm}<{\centering} p{1.1cm}<{\centering} p{1.1cm}<{\centering} p{1.1cm}<{\centering} p{1.1cm}<{\centering} p{1.1cm}<{\centering} p{1.1cm}<{\centering}}        
	    \toprule
	    Dataset	& Model & Top-1 & Top-5 & FLOPs & Pruned & Time & Pruned \\
		\midrule 
		\multirow{4}{*}{ImageNet}
	    & Ours ($t = 0.06$) & \textbf{0.08\%} & \textbf{0.07\%} & 1472B & 29.97\% & 0.429s & 24.74\% \\
        & Ours ($t = 0.08$) & -0.48\% & -0.19\% & 1262B & 39.99\% & 0.391s & 31.40\% \\
        & Ours ($t = 0.11$) & -1.15\% & -0.57\% & 1052B & 49.99\% & 0.362s & 36.67\% \\
        & Ours ($t = 0.13$) & -2.66\% & -1.43\% & \textbf{839B} & \textbf{60.03\%} & \textbf{0.331s} & \textbf{41.93\%} \\
		\bottomrule \\
	\end{tabular}
\label{tab:performance}
\end{table*}

%%%%%%%%%%%%%%%%%%%%%%%%%%%%%%%%%%%%%%%%%%%%%%%%%%%%%%%%%%%%%%%%%%%%

\section{Experiments and analysis}
\label{sec:experiments}

We empirically demonstrate the effectiveness of our similarity-based channel pruning algorithm on two representative CNN architectures VGGNet~\cite{vgg} and ResNet~\cite{resnet}. Results are reported on benchmark datasets CIFAR~\cite{cifar} and ImageNet~\cite{imagenet}. Several state-of-the-art channel pruning algorithms are introduced for performance comparison, including SSL~\cite{wen2016learning}, RDF~\cite{roychowdhury2017reducing}, PFGM~\cite{he2019pruning}, TN~\cite{luo2017thinet}, CP~\cite{he2017channel}, NS~\cite{liu2017learning}, SFP~\cite{he2018soft}.

% including Structured Sparsity Learning (SSL)~\cite{wen2016learning}, Network Slimming (NS)~\cite{liu2017learning}, Pruning Filters via Geometric Median (PFGM)~\cite{he2019pruning}, ThiNet (TN)~\cite{luo2017thinet}, Channel Pruning (CP)~\cite{he2017channel}, Reducing Duplicate Filters (RDF)~\cite{roychowdhury2017reducing}, Soft Filter Pruning (SFP)~\cite{he2018soft}

% \textbf{Implementation details.} 
% On CIFAR dataset, we make use of a variant of VGG-16~\cite{liu2017learning} and a 3-stage pre-activation ResNet~\cite{he2016identity}. On ImageNet dataset, a 4-stage ResNet-50 with bottleneck structure~\cite{resnet} is adopted. Our algorithm is implemented in Pytorch~\cite{pytorch}. For NS, we take its original implementation in Pytorch. For SSL and RDF, we re-implement them in Pytorch by following their original papers. For TN, PFGM, CP and SFP, we directly take their compression results from the literature.

% \textbf{Pruning details.}
% While channel-level pruning is straightforward to realize for single-branch CNN models like AlexNet~\cite{alexnet} and VGGNet, special concerns are required for more sophisticated architectures with cross-layer connections, such as ResNet and DenseNet~\cite{densenet}, to ensure the consistency of forward and backward propagation. To achieve this purpose while not compromising the flexibility of pruning operation over certain layers, we exploit the strategy of inserting a channel selection layer in each residual block, as proposed in~\cite{liu2017learning}. 

\textbf{Implementation details.}
For all experiments on CIFAR-10 and CIFAR-100 datasets, we fine-tune the pruned models using SGD optimizer and batch size 64 for 60 epochs. Learning rate begins at 0.01, decays by 10 at 25, 50 epoch for ResNet and 30, 50 epoch for VGGNet, respectively. On ImageNet dataset, all training settings are kept the same except for a batch size 256. A weight decay of $10^{-4}$ and a Nesterov momentum of 0.9 are utilized. More details can be found in Appendix A.

\textbf{Acceleration metric.}
We adopt FLOPs as the acceleration metric to evaluate all pruning algorithms in our experiments, as it plays a decisive role on the network's inference speed. Different from most previous works, we count all floating point operations that take place during the inference phase when computing overall FLOPs, not only those related to convolution operations, since non-tensor layers (e.g. BN, ReLU and pooling layers) also occupy a considerable part of inference time on GPU~\cite{luo2017thinet}. Additionally, tensor multiply-adds counts for two FLOPs. Reported results of all baseline methods are computed using their publicly available configuration files under our settings. 

%%%%%%%%%%%%%%%%%%%%%%%%%%%%%%%%%%%%%%%%%%%%%%%%%%%%%%%%%%%%%%%%%%%%

\begin{figure}[t!]
\centering
\includegraphics[width=0.9\linewidth]{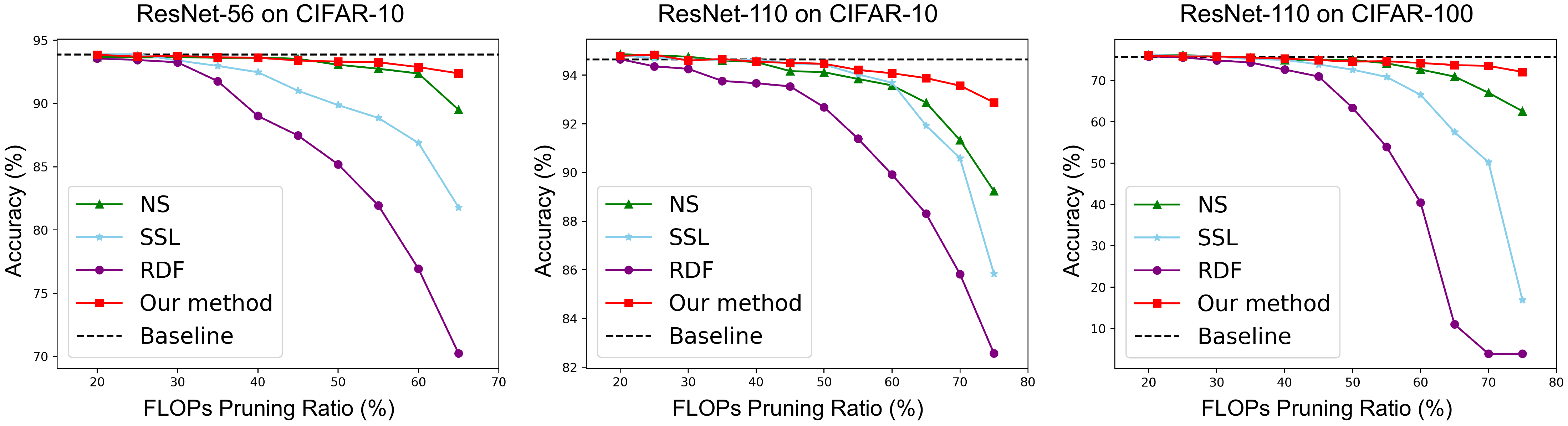}
\caption{Accuracy-acceleration curve on CIFAR dataset. Best viewed in color.}
\label{fig:performance-curve}
\end{figure}

%%%%%%%%%%%%%%%%%%%%%%%%%%%%%%%%%%%%%%%%%%%%%%%%%%%%%%%%%%%%%%%%%%%%

\begin{figure}[t!]
\centering
\includegraphics[width=0.99\linewidth]{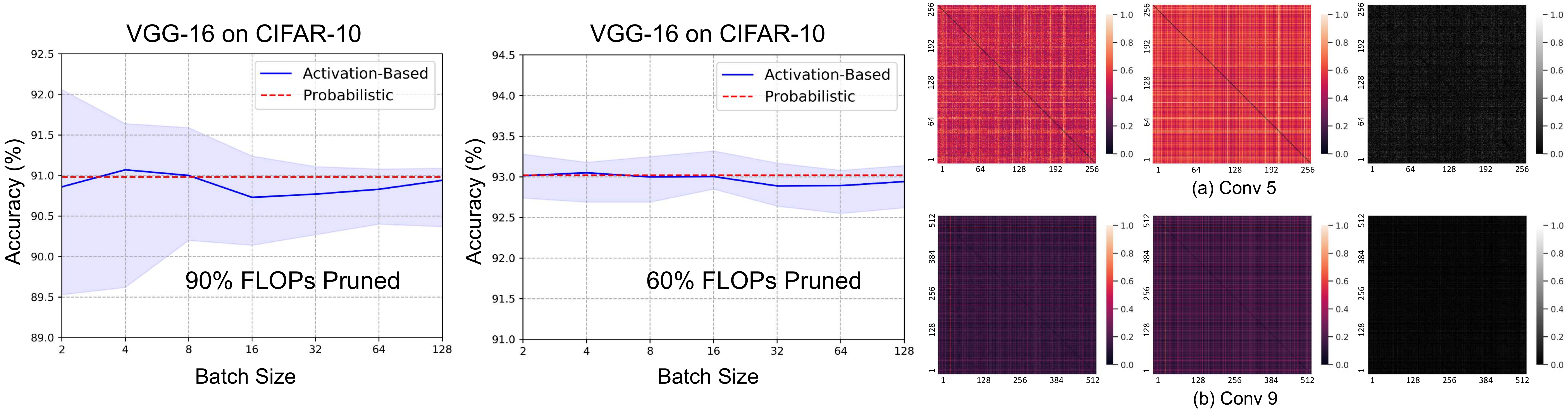}
\caption{Left: Performance comparison between activation-value-based approach and its probabilistic counterpart. For each batch size, 20 trials are conducted using randomly sampled batches. Blue line and shaded region represent the mean accuracy and accuracy range of the former. The accuracy of the latter is added for comparison (dotted red line). Right: Channel distance matrix obtained using Equation~\ref{eq:distance} and averaged over 20 trials of batch size 256 (leftmost), using Equation~\ref{eq:distance-bn} (middle), and their absolute difference (rightmost) for VGG-16 on CIFAR-10. Best viewed in color.}
\label{fig:estimation}
\end{figure}

%%%%%%%%%%%%%%%%%%%%%%%%%%%%%%%%%%%%%%%%%%%%%%%%%%%%%%%%%%%%%%%%%%%%

\subsection{VGGNet and ResNet on CIFAR}

We first evaluate the acceleration performance of our algorithm on CIFAR dataset. The results are summarized in Table~\ref{tab:performance}. Compared with other algorithms, ours consistently achieves better accuracy while reducing more FLOPs. To showcase the superior performance of our approach over the others under different pruning ratios, we further plot the test accuracy in function of the proportion of pruned FLOPs in Figure~\ref{fig:performance-curve}. While the other methods rival ours in accuracy under relatively low pruning ratio, they begin to experience severe performance decay as the ratio goes up. In contrast, our algorithm reveals very stable behaviors in all architecture-dataset settings and maintains decent performance even under extremely aggressive pruning ratio.

\subsection{ResNet on ImageNet}

We then apply our algorithm to pruning ResNet-50 on ImageNet-2012 to validate its effectiveness on large-scale datasets. As shown in Table~\ref{tab:performance}, our algorithm outperforms all 5 competitors by a notable margin in terms of Top-1 and Top-5 accuracy drop, achieving 45.90\% reduction of FLOPs at the cost of only 0.34\% drop of Top-5 accuracy. We further investigate how our algorithm extends to varying pruning ratios and how this theoretical reduction of computation translates to realistic acceleration on modern GPUs. The results presented in Table~\ref{tab:performance}(c) are obtained using PyTorch~\cite{pytorch} and cuDNN v7.1 on a TITAN X Pascal GPU. The gap between FLOPs saving and time saving can be attributed to the fact that non-computing operations, e.g. IO query and buffer switch, also influence inference speed.  

%%%%%%%%%%%%%%%%%%%%%%%%%%%%%%%%%%%%%%%%%%%%%%%%%%%%%%%%%%%%%%%%%%%%

\subsection{Channel distance estimation}
\label{sec:estimation}

As discussed in Section~\ref{sec:metric}, calculating channel distance via activation-value-based approach, i.e. using Equation~\ref{eq:distance}, is data-dependent and inconsistent across different data batches. To elaborate on this point, we study the performance of pruned VGG-16 model obtained using randomly sampled batches. As illustrated in Figure~\ref{fig:estimation}, while activation-based approach matches probabilistic approach in mean accuracy, its performance fluctuates significantly across different input samples and exhibits highly unstable pattern. In addition, this problem of instability gets worse with higher pruning ratio and smaller batch size. In contrast, probabilistic approach does not rely on input data and consistently achieves satisfactory performance. This result strongly confirms the effectiveness and necessity of the proposed probabilistic approach to channel distance estimation.

In order to validate that our probabilistic approach provides a good approximation to the mathematical expectation of its activation-based counterpart and that the assumption of Proposition~\ref{prop1} causes no bias in the whole estimation process, we further compare the distance matrix computed using Equation~\ref{eq:distance} (averaged over 20 trials) and that computed using Equation~\ref{eq:distance-bn}. As we can see from Figure~\ref{fig:estimation}, there only exists minor discrepancy between the mean distance matrix obtained via activation-based approach and the distance matrix estimated via probabilistic approach, which effectively supports our claim. This observation also explains why the two approaches share similar performance in mean accuracy. More visualization results can be found in Appendix D.

%%%%%%%%%%%%%%%%%%%%%%%%%%%%%%%%%%%%%%%%%%%%%%%%%%%%%%%%%%%%%%%%%%%%

\subsection{Normalized channel distance matrix}
\label{sec:normalization}

To explore the effect of normalizing channel distance matrix before performing HC, we visualize the distribution of values within both the normalized matrix and the unnormalized one in Figure~\ref{fig:normalization}. From the visualization results, we observe that, without distance normalization, value distribution is rather lopsided across different layers. This phenomenon results in imbalanced target architecture when all layers are pruned using a global threshold, i.e. some upper layers are boiled down to a single channel. As we can see from Figure~\ref{fig:normalization}, normalization operation effectively alleviates this dilemma, enabling automatic search of target architecture while ensuring efficient propagation of feature information through the network. Comparison of accuracy curve between the two well corroborates this point.

\subsection{Performance analysis}

We further look into the reason for the superior performance of our algorithm over the other automatic pruning algorithms in finding efficient and compact architectures. As shown in Figure~\ref{fig:performance-histogram}, NS and SSL prune very aggressively over upper layers, making it extremely hard to propagate feature information extracted in lower layers up to the classification layer. Moreover, upper layers generate very limited FLOPs as the feature maps there are already down-sampled several times and thus of smaller size, explaining why our algorithm retains more channels while pruning the same amount of FLOPs.  

%%%%%%%%%%%%%%%%%%%%%%%%%%%%%%%%%%%%%%%%%%%%%%%%%%%%%%%%%%%%%%%%%%%%

\begin{figure}[t!]
\centering
\includegraphics[width=0.9\linewidth]{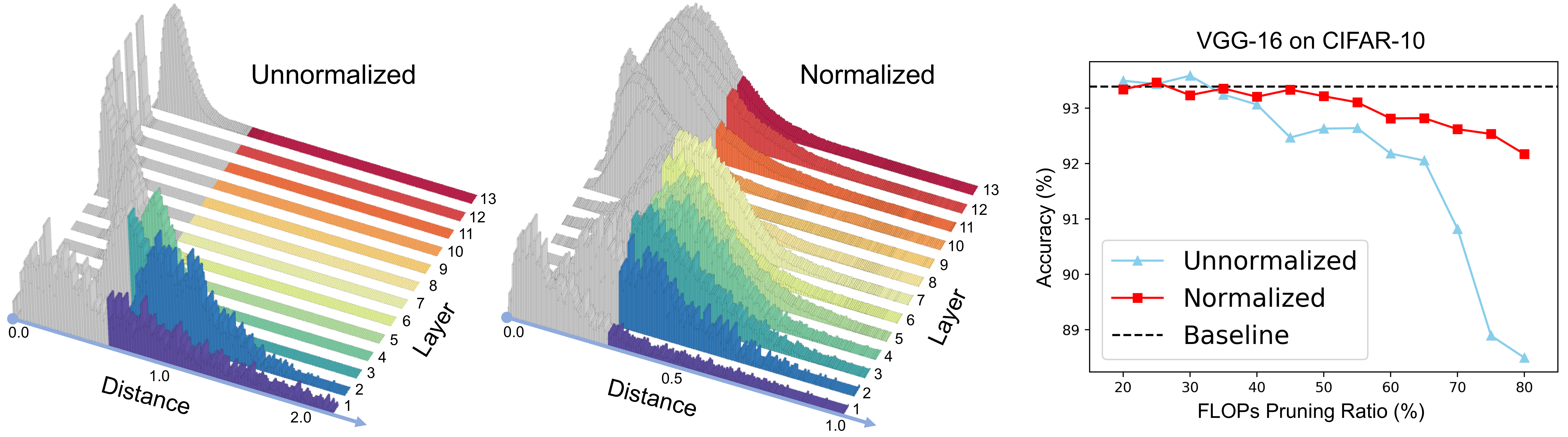}
\caption{Left: Distribution of values in both the unnormalized distance matrix and the normalized one for VGG-16 on CIFAR-10. Gray area shows the distance values below the pruning threshold, i.e. those pairs of channels that should be grouped into clusters. Right: Performance comparison between the two cases on CIFAR-10. Best viewed in color.}
\label{fig:normalization}
\end{figure}

%%%%%%%%%%%%%%%%%%%%%%%%%%%%%%%%%%%%%%%%%%%%%%%%%%%%%%%%%%%%%%%%%%%%

\begin{figure}[t!]
\centering
\includegraphics[width=0.9\linewidth]{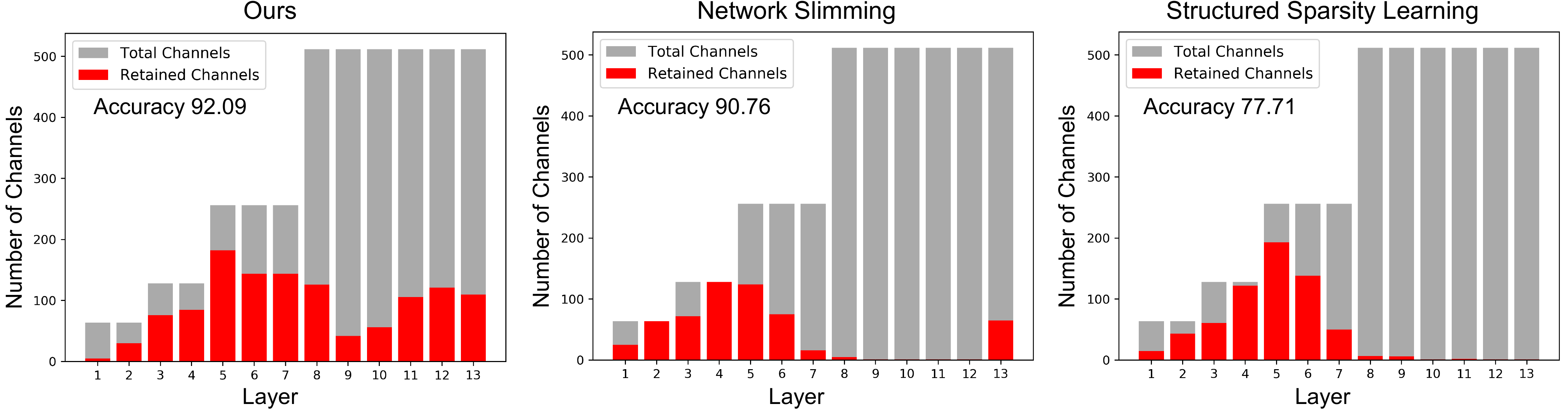}
\caption{Architecture of pruned VGG-16 model with 80\% reduction of FLOPs on CIFAR-10 dataset. Two automatic pruning algorithms, SSL and NS, are introduced for comparison. Best viewed in color.}
\label{fig:performance-histogram}
\end{figure}

%%%%%%%%%%%%%%%%%%%%%%%%%%%%%%%%%%%%%%%%%%%%%%%%%%%%%%%%%%%%%%%%%%%%

\section{Conclusion  }

We propose a novel perspective to perform channel-level pruning and accelerate deep CNNs by showing that channels revealing similar feature information have functional overlap and that most channels within each such similarity group can be removed with minor impact on network's representational power. Experimental results well support the reasonability of our intuition. In the future, we will try extending our approach to more general neural network models, such as RNNs and GNNs.

%%%%%%%%%%%%%%%%%%%%%%%%%%%%%%%%%%%%%%%%%%%%%%%%%%%%%%%%%%%%%%%%%%%%

% \clearpage
\printbibliography[heading=bibintoc]

%%%%%%%%%%%%%%%%%%%%%%%%%%%%%%%%%%%%%%%%%%%%%%%%%%%%%%%%

% \clearpage
\section*{A. \hspace{1mm} Implementation}

\textbf{Implementation details.} 
On CIFAR dataset, we make use of a variant of VGG-16~\cite{liu2017learning} and a 3-stage pre-activation ResNet~\cite{he2016identity}. On ImageNet dataset, a 4-stage ResNet-50 with bottleneck structure~\cite{resnet} is adopted. Our algorithm is implemented in Pytorch~\cite{pytorch}. For NS, we take its original implementation in Pytorch. For SSL and RDF, we re-implement them in Pytorch by following their original papers. For TN, PFGM, CP and SFP, we directly take their acceleration results from the literature.

\textbf{Pruning details.}
While channel-level pruning is straightforward to realize for single-branch CNN models like AlexNet~\cite{alexnet} and VGGNet, special concerns are required for more sophisticated architectures with cross-layer connections, such as ResNet and DenseNet~\cite{densenet}, to ensure the consistency of forward and backward propagation. To achieve this purpose while not compromising the flexibility of pruning operation over certain layers, we exploit the strategy of inserting a channel selection layer in each residual block, as proposed in~\cite{liu2017learning}.

%%%%%%%%%%%%%%%%%%%%%%%%%%%%%%%%%%%%%%%%%%%%%%%%%%%%%%%%%

\section*{B. \hspace{1mm} Proof}

\begin{prop}
\label{prop1copy}
Assume that the activations belonging to the same channel are i.i.d. and that any two activations from two different channels are mutually independent. For any two channels in the same convolutional layer, the distance between them converges in probability to the following value:
\begin{equation}
\label{eq:prop1}
\frac{1}{n} \lVert \mathcal{A}^{(i)} - \mathcal{A}^{(j)} \rVert_{2}^{2} \xrightarrow[n \to \infty]{proba} ( \, \mu^{(i)} - \mu^{(j)})^{2} + (\sigma^{(i)})^{2} + (\sigma^{(j)})^{2} \quad \text{where} \, n = H \times W \times B
\end{equation}
\end{prop}

%%%%%%%%%%%%%%%%%%%%%%%%%%%%%%%%%%%%%%%%%%%%%%%%%%%%%%%%%

\begin{proof}
The proof of this proposition is a direct application of the weak law of large numbers.
Since the activations within the same feature channel are i.i.d., we can denote $\mathcal{A}_{h, w, b}^{(c)} \sim \mathcal{D}^{(c)}( \, \mu^{(c)}, (\sigma^{(c)})^{2})$. In addition, if we define that $\mathcal{Z}_{h, w, b}^{(i, j)} = ( \, \mathcal{A}_{h, w, b}^{(i)} - \mathcal{A}_{h, w, b}^{(j)} \, )^{2} \sim \mathcal{D}^{(i, j)}$, $\hat{\mathcal{A}}^{(i)} \sim \mathcal{D}^{(i)}( \, \mu^{(i)}, (\sigma^{(i)})^{2})$ and $\hat{\mathcal{A}}^{(j)} \sim \mathcal{D}^{(j)}( \, \mu^{(j)}, (\sigma^{(j)})^{2})$ are mutually independent, and $\hat{\mathcal{Z}}^{(i, j)} = ( \, \hat{\mathcal{A}}^{(i)} - \hat{\mathcal{A}}^{(j)} \, )^{2}$, then we have:
\begin{equation}
\label{eq:lln}
\begin{array}{rl}
& P \Big\{ \Big| \frac{1}{n} \lVert \mathcal{A}^{(i)} - \mathcal{A}^{(j)} \rVert_{2}^{2} - \mathbb{E}[(\hat{\mathcal{A}}^{(i)} - \hat{\mathcal{A}}^{(j)})^{2}] \Big| \le \varepsilon \Big\} \\ [3mm]
= & P \Big\{ \Big| \frac{1}{n} \sum_{h, w, b} (\mathcal{A}_{h, w, b}^{(i)} - \mathcal{A}_{h, w, b}^{(j)})^{2} - \mathbb{E}[(\hat{\mathcal{A}}^{(i)} - \hat{\mathcal{A}}^{(j)})^{2}] \Big| \le \varepsilon \Big\} \\ [3mm]
= & P \Big\{ \Big| \frac{1}{n} \sum_{h, w, b} \mathcal{Z}_{h, w, b}^{(i, j)} - \mathbb{E}[\hat{\mathcal{Z}}^{(i, j)}] \Big| \le \varepsilon \Big\}
\end{array}
\end{equation}
The sequence of random variables $\mathcal{Z}_{h, w, b}^{(i, j)}$ are i.i.d. and $\mathbb{E}[\hat{\mathcal{Z}}^{(i, j)}] = \mathbb{E}[\mathcal{Z}_{h, w, b}^{(i, j)}] \hspace{0.2cm} \forall h, w, b$. We can thus apply the weak law of large numbers to Equation~\ref{eq:lln}:
\begin{equation}
\begin{array}{rcl}
\lim_{n \to \infty} P \Big\{ \Big| \frac{1}{n} \sum_{h, w, b} \mathcal{Z}_{h, w, b}^{i, j} - \mathbb{E}[\hat{\mathcal{Z}}^{(i, j)}] \Big| \le \varepsilon \Big\} & = & 1
\end{array}
\end{equation}
Note that $\hat{\mathcal{A}}^{(i)}$ and $\hat{\mathcal{A}}^{(j)}$ are mutually independent, therefore the r.h.s. of Equation~\ref{eq:prop1} can be derived by using the first two moments (mean and variance) of $\hat{\mathcal{A}}^{(i)}$ and $\hat{\mathcal{A}}^{(j)}$ as follows:
\begin{equation}
\begin{array}{rcl}
\mathbb{E}[(\hat{\mathcal{A}}^{(i)} - \hat{\mathcal{A}}^{(j)})^{2}] & = & \mathbb{E}[(\hat{\mathcal{A}}^{(i)})^{2}] + \mathbb{E}[(\hat{\mathcal{A}}^{(j)})^{2}] - \mathbb{E}[2 \, \hat{\mathcal{A}}^{(i)} \, \hat{\mathcal{A}}^{(j)}] \\ [3mm]
& = & \mathbb{V}ar[\hat{\mathcal{A}}^{(i)}] + \mathbb{E}[\hat{\mathcal{A}}^{(i)}]^{2} + \mathbb{V}ar[\hat{\mathcal{A}}^{(j)}] + \mathbb{E}[\hat{\mathcal{A}}^{(j)}]^{2} - 2 \, \mathbb{E}[\hat{\mathcal{A}}^{(i)}] \, \mathbb{E}[\hat{\mathcal{A}}^{(j)}] \\ [3mm]
& = & ( \, \mu^{(i)} - \mu^{(j)})^{2} + (\sigma^{(i)})^{2} + (\sigma^{(j)})^{2}
\end{array}
\end{equation}
This concludes the proof of Proposition~\ref{prop1copy}. 
\end{proof}
\vskip 1cm

%%%%%%%%%%%%%%%%%%%%%%%%%%%%%%%%%%%%%%%%%%%%%%%%%%%%%%%%%

\begin{prop}
\label{prop2copy}
For each feature channel $\mathcal{A}^{(l+1, c_{l+1})}$ in the $(l+1)$-th convolutional layer, the distance shift caused by removing the feature channel $\mathcal{N}^{(l, i)}$ from the $l$-th convolutional layer, as defined in Equation~\ref{eq:shift}, admits the following upper bound:
\begin{equation}
\label{eq:bound}
Dist( \, \mathcal{A}^{(l+1, c_{l+1})}, \mathcal{A}_{p}^{(l+1, c_{l+1})}) \le \lambda \times \hspace{-3mm} \min_{j \in \{1, ..., C_{l}\}} Dist( \, \mathcal{N}^{(l, i)}, \mathcal{N}^{(l, j)}) 
\end{equation}
where $\lambda = \frac{n_{l}}{n_{l+1}} K^{2} \lVert \mathcal{W}^{(i, c_{l+1})} \rVert_{2}^{2}$ and $K^{2}$ corresponds to the size of each kernel matrix $\mathcal{W}^{(c_{l}, c_{l+1})}$.
\end{prop}

%%%%%%%%%%%%%%%%%%%%%%%%%%%%%%%%%%%%%%%%%%%%%%%%%%%%%%%%%

\begin{proof}
For an arbitrary feature channel $\mathcal{N}^{(l, j)}$ in the $l$-th convolutional layer, let $\delta_{i, j}^{(l)} = h(\mathcal{N}^{(l, i)}) - h(\mathcal{N}^{(l, j)})$ and $\delta_{i, j}^{(l)}( \, p_{1}, p_{2})$ denote the image patch centered at $( \, p_{1}, p_{2})$. Note that $\delta_{i, j}^{(l)}( \, p_{1}, p_{2})$ is convolved with $\mathcal{W}^{(i, c_{l+1})}$ to generate the activation in the $p_{1}$-th row and $p_{2}$-th column in the $c_{l+1}$-th output feature channel. We transform the l.h.s of Equation~\ref{eq:bound} as follows:
\begin{equation}
\label{eq:cauchy}
\begin{array}{rcl}
& & Dist( \, \mathcal{A}^{(l+1, c_{l+1})}, \mathcal{A}_{p}^{(l+1, c_{l+1})}) \\ [2mm]
& = & \frac{1}{n_{l+1}} \lVert (h(\mathcal{N}^{(l, i)}) - h(\mathcal{N}^{(l, j)})) \ast \mathcal{W}^{(i, c_{l+1})} \rVert_{2}^{2} \\ [3mm]
& = & \frac{1}{n_{l+1}} \sum_{\substack{p_{1} \in \{1, ..., H_{l+1}\} \\ p_{2} \in \{1, ..., W_{l+1}\}}} (\delta_{i, j}^{(l)}( \, p_{1}, p_{2}) \ast \mathcal{W}^{(i, c_{l+1})})^{2} \\ [4mm]
& = & \frac{1}{n_{l+1}} \sum_{\substack{p_{1} \in \{1, ..., H_{l+1}\} \\ p_{2} \in \{1, ..., W_{l+1}\}}} | \langle \delta_{i, j}^{(l)}( \, p_{1}, p_{2}), \mathcal{W}^{(i, c_{l+1})} \rangle |^{2}
\end{array}
\end{equation}
Note that the linear convolution between $\delta_{i, j}^{(l)}( \, p_{1}, p_{2})$ and $\mathcal{W}^{(i, c_{l+1})}$, both of size $K \times K$, is equivalent to their inner product. Applying Cauchy-Schwarz inequality to Equation~\ref{eq:cauchy} gives:
\begin{equation}
\label{eq:temp}
\begin{array}{rcl}
& & Dist( \, \mathcal{A}^{(l+1, c_{l+1})}, \mathcal{A}_{p}^{(l+1, c_{l+1})}) \\ [2mm]
& \le & \frac{1}{n_{l+1}} \sum_{\substack{p_{1} \in \{1, ..., H_{l+1}\} \\ p_{2} \in \{1, ..., W_{l+1}\}}} \lVert \delta_{i, j}^{(l)}( \, p_{1}, p_{2}) \rVert_{2}^{2} \, \lVert \mathcal{W}^{(i, c_{l+1})} \rVert_{2}^{2} \\ [4mm]
& = & \frac{1}{n_{l+1}} \lVert \mathcal{W}^{(i, c_{l+1})} \rVert_{2}^{2} \, \sum_{\substack{p_{1} \in \{1, ..., H_{l+1}\} \\ p_{2} \in \{1, ..., W_{l+1}\}}} \lVert \delta_{i, j}^{(l)}( \, p_{1}, p_{2}) \rVert_{2}^{2} \\ [4mm]
& \le & \frac{1}{n_{l+1}} K^{2} \lVert \mathcal{W}^{(i, c_{l+1})} \rVert_{2}^{2} \, \lVert \delta_{i, j}^{(l)} \rVert_{2}^{2}
\end{array}
\end{equation}
The last inequality stems from the fact that each activation in $\delta_{i, j}^{(l)}$ appears at most $K^{2}$ times in the convolution operation. Actually, most activations in $\delta_{i, j}^{(l)}$ participate exactly $K^{2}$ times in the convolution operation, except for those lying near the border of $\delta_{i, j}^{(l)}$, which appear less often.

For the two most popular choices of non-linear activation functions in modern CNN architectures, sigmoid and ReLU, we have the extra property that $\max_{x \in \mathbb{R}}(\frac{{\rm d}h(x)}{{\rm d}x}) \leq 1$ and $\min_{x \in \mathbb{R}}(\frac{{\rm d}h(x)}{{\rm d}x}) \geq 0$, based on which we can easily derive the following inequality:
\begin{equation}
\label{eq:activation}
(h(x_{1}) - h(x_{2}))^{2} \le (x_{1} - x_{2})^{2} \quad \forall x_{1}, x_{2} \in \mathbb{R}
\end{equation}
Based on the conclusion of Equation~\ref{eq:activation}, we transform the inequality in Equation~\ref{eq:temp} as follows:
\begin{equation}
\begin{array}{rcl}
& & Dist( \, \mathcal{A}^{(l+1, c_{l+1})}, \mathcal{A}_{p}^{(l+1, c_{l+1})}) \\ [2mm]
& \le & \frac{1}{n_{l+1}} K^{2} \lVert \mathcal{W}^{(i, c_{l+1})} \rVert_{2}^{2} \, \lVert h(\mathcal{N}^{(l, i)}) - h(\mathcal{N}^{(l, j)}) \rVert_{2}^{2} \\ [3mm]
& \le & \frac{1}{n_{l+1}} K^{2} \lVert \mathcal{W}^{(i, c_{l+1})} \rVert_{2}^{2} \, \lVert \mathcal{N}^{(l, i)} - \mathcal{N}^{(l, j)} \rVert_{2}^{2} \\ [3mm]
& = & \frac{n_{l}}{n_{l+1}} K^{2} \lVert \mathcal{W}^{(i, c_{l+1})} \rVert_{2}^{2} \, Dist( \, \mathcal{N}^{(l, i)}, \mathcal{N}^{(l, j)})
\end{array}    
\end{equation}
Since $\mathcal{N}^{(l, j)}$ can be whichever feature channel in the $l$-th convolutional layer, we can further narrow the upper bound by taking the minimum over all $j \in \{1, ..., C_{l}\}$.
\begin{equation}
\begin{array}{rcl}
& & Dist( \, \mathcal{A}^{(l+1, c_{l+1})}, \mathcal{A}_{p}^{(l+1, c_{l+1})}) \\ [3mm] 
& \le & \frac{n_{l}}{n_{l+1}} K^{2} \lVert \mathcal{W}^{(i, c_{l+1})} \rVert_{2}^{2} \, \min_{j \in \{1, ..., C_{l}\}} Dist( \, \mathcal{N}^{(l, i)}, \mathcal{N}^{(l, j)}) \\ [2mm]
& = & \lambda \times \min_{j \in \{1, ..., C_{l}\}} Dist( \, \mathcal{N}^{(l, i)}, \mathcal{N}^{(l, j)}) 
\end{array}
\end{equation}
where $\lambda = \frac{n_{l}}{n_{l+1}} K^{2} \lVert \mathcal{W}^{(i, c_{l+1})} \rVert_{2}^{2}$ \hspace{1pt} and \hspace{1pt} $n_{l} = H_{l} \times W_{l} \times B$ (resp. $n_{l+1}$).

This concludes the proof of proposition~\ref{prop2copy}. 
\end{proof}

%%%%%%%%%%%%%%%%%%%%%%%%%%%%%%%%%%%%%%%%%%%%%%%%%%%%%%%%%%%%%%%%%%%%

\section*{C. \hspace{1mm} Pseudo code}
\label{appendixC}

Pseudo code of the proposed channel-similarity-based pruning algorithm.

{\centering
\begin{minipage}{.68\linewidth}
\centering
\begin{algorithm}[H]
\caption{Channel-similarity-based pruning via hierarchical clustering of feature channels}
\label{al:method}
\begin{algorithmic}
    \STATE {\bfseries Input:} Threshold distance $t$ and feature channels before pruning $\mathcal{N} = \{\mathcal{N}^{(l, c_{l})} \, | \, c_{l} \in \{1, ..., C_{l}\}, l \in \{1, ..., L\}\}$
    \STATE {\bfseries Output:} Retained channels $\mathcal{N}_{r} = \{\mathcal{N}_{r}^{(l)} \, | \, l \in \{1, ..., L\}\}$
    \vskip 3mm
    \FOR{$l = 1$ {\bfseries to} $L$}
    \vskip 1mm
        \STATE {\bfseries Initialize} $\mathcal{D}^{(l)} \in \mathbb{R}^{C_{l} \times C_{l}}$ and $\mathcal{N}_{r}^{(l)} = \emptyset$
        \vskip 1mm
        \FOR{$i = 1$ {\bfseries to} $C_{l}$}
        \vskip 1mm
            \FOR{$j = i + 1$ {\bfseries to} $C_{l}$}
            \vskip 1mm
                \STATE $\mathcal{D}^{(l)}( \, i, j) \leftarrow Dist( \, \mathcal{N}^{(l, i)}, \mathcal{N}^{(l, j)})$ \\ 
                \hspace{3.5cm} \# Equation~4 of the main paper
                \STATE $\mathcal{D}^{(l)}( \, j, i) \leftarrow  \mathcal{D}^{(l)}( \, i, j)$
                \vskip 1mm
            \ENDFOR
            \vskip 1mm
        \ENDFOR
        \vskip 1mm
        \STATE $\mathcal{D}_{min}^{(l)} \leftarrow \min_{i, j} \mathcal{D}^{(l)}( \, i, j)$ and $\mathcal{D}_{max}^{(l)} \leftarrow \max_{i, j} \mathcal{D}^{(l)}( \, i, j)$
        \vskip 1mm
        \STATE Normalization: $\mathcal{D}^{(l)} \leftarrow (\mathcal{D}^{(l)} - \mathcal{D}_{min}^{(l)}) \times (\mathcal{D}_{max}^{(l)} - \mathcal{D}_{min}^{(l)})^{-1}$
        \vskip 1mm
        \STATE Perform hierarchical clustering with threshold $t$ on $\mathcal{N}^{(l)}$ using $\mathcal{D}^{(l)}$ to obtain $\mathcal{G}^{(l)} = \{\mathcal{G}^{(l, m)} \, | \, m \in \{1, ..., M\}\}$
        \vskip 1mm
        \FOR{$m = 1$ {\bfseries to} $M$}
        \vskip 1mm
            \STATE $\mathcal{N}_{m}^{(l)} \leftarrow \arg\max_{\mathcal{N}^{(l, c_{l})} \in \mathcal{G}^{(l, m)}} |\gamma^{(c_{l})}|$
            \vskip 1mm
            \STATE $\mathcal{N}_{r}^{(l)} \leftarrow \mathcal{N}_{r}^{(l)} \, \cup \, \mathcal{N}_{m}^{(l)}$
            \vskip 1mm
        \ENDFOR
        \vskip 1mm
    \ENDFOR
\end{algorithmic}
\end{algorithm}
\end{minipage}
\par}

%%%%%%%%%%%%%%%%%%%%%%%%%%%%%%%%%%%%%%%%%%%%%%%%%%%%%%%%%%%%%%%%%%%%

\section*{D. \hspace{1mm} Visualization}
\label{appendixD}

We visualize the channel distance matrix computed via the activation-value-based approach (averaged over 20 trials using randomly sampled image batches) and that estimated via the proposed probabilistic approach for a number of convolutional layers of VGG-16 on CIFAR-10 in Figure~\ref{fig:distance-matrix}. Their absolute difference is added for more visually accessible illustration.  

\begin{figure}[h]
\centering
\includegraphics[width=0.99\linewidth]{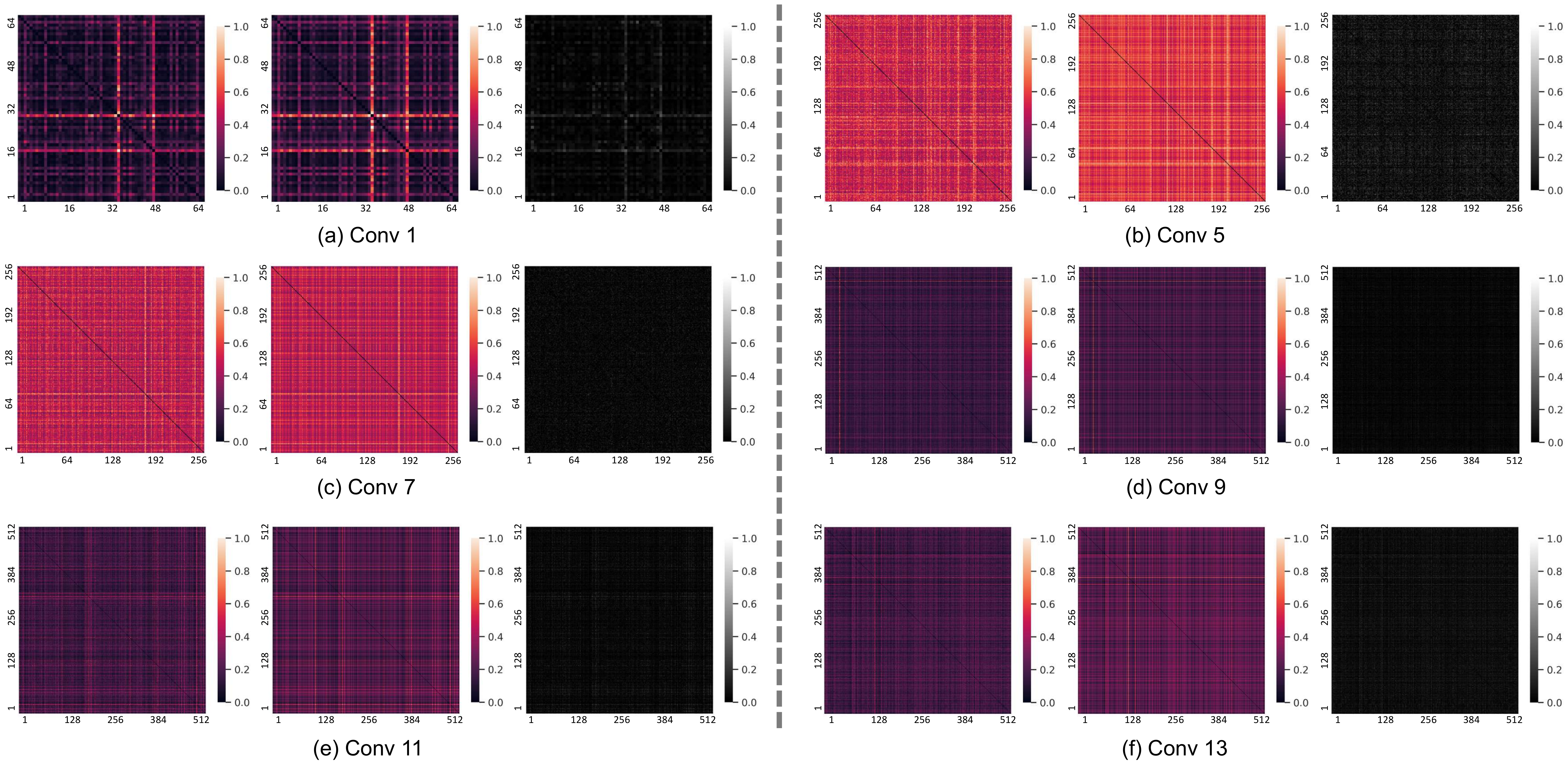}
\caption{Visualization of the channel distance matrix computed using different approaches. The leftmost figure of each triplet shows the matrix obtained via activation-value-based approach (averaged over 20 trials of batch size 256), the middle shows that estimated via probabilistic approach, while the rightmost shows their absolute difference. Best viewed in color.}
\label{fig:distance-matrix}
\end{figure}

%%%%%%%%%%%%%%%%%%%%%%%%%%%%%%%%%%%%%%%%%%%%%%%%%%%%%%%%%%%%%%%%%%%%

\end{document}